\documentclass{article}

\usepackage{epsfig, graphics}
\usepackage{latexsym,subfigure}
\usepackage{amsmath,amssymb,amsthm,amsfonts}
\usepackage[numbers]{natbib}
\usepackage{algorithm,algorithmic}

\def\sgn{\text{Sign}}

\newtheorem{conjecture}{Conjecture}
\newtheorem{lemma}{Lemma}
\newtheorem{theorem}{Theorem}

\newtheorem{corollary}{Corollary}

\def\hati{\widehat{i}}
\def\hatj{\widehat{j}}
\def\hatm{\widehat{m}}

\def\Os{\widehat{\Omega}_s}
\def\Ob{\widehat{\Omega}_b}

\def\L{\mathcal{L}}
\def\real{\mathbb{R}}

\def\Loss{\mathcal{L}}

\def\delpar{\Delta}
\def\delparHat{\widehat{\Delta}}
\def\grad{\nabla}
\def\noiseLevel{\lambda}

\newcommand{\tr}[2]{\left<#1,#2\right>}

\def\L{\mathcal{L}}

\def\real{\mathbb{R}}

\title{A New Greedy Algorithm for Multiple Sparse Regression}

\author{Ali Jalali\\
Yahoo! Labs\\
\texttt{ajalali@yahoo-inc.com}\\
\and
Sujay Sanghavi\\
University of Texas at Asutin\\
\texttt{sanghavi@mail.utexas.edu}}

\begin{document}
\maketitle

\begin{abstract}
This paper proposes a new algorithm for  multiple sparse regression  in high dimensions, where the task is to estimate the support and values of several (typically related) sparse vectors from a few noisy linear measurements. Our algorithm is a ``forward-backward" greedy procedure that -- uniquely -- {\em operates on two distinct classes of objects.} In particular, we organize our target sparse vectors as a matrix; our algorithm involves iterative addition and removal of both {\em (a)} individual elements, and {\em (b)} entire rows (corresponding to shared features), of the matrix. 

Analytically, we establish that our algorithm manages to recover the supports (exactly) and values (approximately) of the sparse vectors, under assumptions similar to existing approaches based on convex optimization. However, our algorithm has a much smaller computational complexity. Perhaps most interestingly, it is seen empirically to require visibly fewer samples. Ours represents the first attempt to extend greedy algorithms to the class of models that can only/best be represented by a {\em combination} of component structural assumptions (sparse and group-sparse, in our case).
\end{abstract}

\section{Introduction} 
This paper provides a new algorithm for the (standard) {\bf multiple sparse linear regression} problem, which we now describe. We are interested in inferring $r$ sparse vectors $\beta^{*(1)},\ldots,\beta^{*(r)}\in\real^p$ from noisy linear measurements; in particular, for each  $1\leq j \leq r$, we observe $n_j$ noisy linear measurements according to the statistical model
\begin{equation}
y^{(j)} = X^{(j)}\beta^{*(j)} + z^{(j)}\qquad\qquad \forall j\in\{1,\ldots,r\},
\end{equation}
where for each $j$, $X^{(j)}\in\real^{n_j\times p}$ is the design matrix, $y^{(j)}\in\real^{n_j}$ is the response vector and $z^{(j)}\in\real^{n_j}$ is the noise. We combine all tasks $\beta^{*(j)}$ as columns of a matrix $\beta^*\in\real^{p\times r}$. We are thus interested in inferring the matrix $\beta^*$ given $(y^{(j)},X^{(j)})$, for $1\leq j \leq r$.  Here inference means both recovery of the support of $\beta^*$, as well as closeness in numerical values on the non-zero elements.

We are interested in solving this problem in the {\bf high-dimensional setting}, where the number of observations $n_j$ is potentially substantially smaller than the number of features $p$. High-dimensional settings arise in applications where measurements are expensive, and hence a sufficient number may be unavailable. Consistent recovery of $\beta^*$ is now not possible in general; however, as is now well-recognized, it is possible if each $\beta^{*(j)}$ is sparse, and the design matrices satisfy certain properties.

Multiple sparse linear regression comes up in applications ranging from graphical model selection \cite{RWLIsing} and kernel learning \cite{BachMKL} to function estimation \cite{RLLWSPAM} and multi-task learning \cite{JRSR10}, etc. In several of these examples, the different $\beta^{*(j)}$ vectors are {\em related}, in the sense that they share portions of their supports/features, and may even be close in values on those entries. As an example, consider the task of learning handwritten character ``A" for different writers. Since all these handwritings read ``A", they should share a lot of features, but of course there might be few non-shared features indicating each individual handwriting. A natural question in this setting is: can inferring the vectors {\bf jointly} (often referred to as multi-task learning \cite{CaruanaMTL}) result in lower sample complexity than inferring each one individually? 

When the sharing of supports is partial, it turns out the answer depends on the method used. Some ``group LASSO" methods like $\ell_1/\ell_q$ regularization can actually result in {\em lower or higher} sample complexity, as compared to doing for example separate LASSO, depending on whether the level of sharing among tasks is high or low, respectively. The ``dirty mode" approach \cite{JRSR10} develops a method, based on splitting $\beta^*$ into two matrices which are regularized differently, which shows gains in sample complexity for all levels of sharing. We review the related existing work in section \ref{sec:related}.

{\bf Our Contribution:} We provide a novel forward-backward greedy algorithm, designed for when the target structure is a combination of a sparse and block-sparse matrix. We provide theoretical guarantee on the performance of the algorithm in terms of both estimation error and support recovery. Our analysis is more subtle than \cite{JJR11}, since we would like to have \emph{local} assumptions on each task $\beta^{*(j)}$ as opposed to having \emph{global} assumptions on the whole matrix $\beta^*$. Ours is the first attempt to extend greedy approaches, which are sometimes seen to be both statistically and computationally more efficient than convex programming, to high-dimensional problems where the best/only approach involves the use of more than one structural models (sparse and group-sparse in our case). 

\subsection{Related Work} \label{sec:related}

There is now a huge literature on sparse recovery from linear measurements; we restrict ourselves here to the most directly related work on multiple sparse linear regression.

{\bf Convex optimization approaches:} A popular recent approach to leverage sharing in sparsity patterns has been via the use of  $\ell_1/\ell_q$ group norms as regularizers, with $q>1$; examples include the $\ell_1/\ell_\infty$ norm \citep{TurlachVW, ZH08, NWJoint}, and the $\ell_1/\ell_2$ norm~\citep{Lounici09,Obozinski10}. The sample complexity of these methods  may be sensitive \citep{NWJoint} to the level of sharing of supports, motivating the ``dirty model" approach \citep{JRSR10}; in that paper, the unknown matrix was split as the sum of two matrices, regularized to encourage group-sparsity in one and sparsity in the other. Conceptually, this is similar to our line of thinking; however, their approach was based on convex optimization. We show that our method empirically has lower sample complexity than \citep{JRSR10} (although we do not have a theoretical characterization of the constant multiplicative factor that seems to be the difference).

{\bf Greedy methods:} Several algorithms attempt to find the support (and hence values) of sparse vectors by iteratively adding, and possibly dropping, elements from the support. The earliest examples were simple ``forward" algorithms like Orthogonal Matching Pursuit (OMP) \cite{Tropp07,Zhang10}, etc.; these add elements to the support until the loss goes below a threshold. More recently, it has been shown \cite{Zhang08,JJR11} that adding a backward step is more statistically efficient, requiring weaker conditions for support recovery. Another line of (forward) greedy algorithms works by looking at the gradient of the loss function, instead of the function itself; see e.g. \cite{jain2011orthogonal}. A big difference between our work and these is that our forward-backward algorithm works with {\em two different classes of objects simultaneously:} singleton elements of the matrix of vectors that need to be recovered, and entire rows of this matrix. This adds a significant extra dimension in algorithm design, as we need a way to compare the gains provided by each class of object in a way that ensures convergence and correctness.

\section{Our Algorithm}
We now first briefly describe the algorithm in words, and then specify it precisely. A natural loss function for our multi-task problem is 
\[
\mathcal{L}(\beta)=\sum_{j=1}^r\,\frac{1}{2n_j}\|y^{(j)}-X^{(j)}\beta^{(j)}\|_2^2
\]
Let $\beta = [\beta^{(1)} \ldots \beta^{(r)}]$ be the $p\times r$ matrix which has the $j^{th}$ target vector $\beta^{(j)}$ as its $j^{th}$ column. Our algorithm is based on iteratively building and modifying the estimated support of $\beta$, by adding (in the forward step) and removing (in the backward step) {\bf two kinds} of objects:
singleton elements $(i,j)$, and entire rows $m$. The {\bf basic idea} is to include in forward steps singletons/rows that give big decreases in the loss, and to remove in the backward steps those whose removal results in only a small increase in the loss. However, the kinds of objects cannot be compared in an ``apples to apples" way, which means that doing the forward and backward steps in a way that ensures convergence and correctness is not immediate; as we will see below there are some intricacies in how the addition and removal decisions are made. 

It is easiest to understand our algorithm in terms of ``reward" for forward steps, and``cost" for backward steps. Each inclusion results in a decrease in the loss; the corresponding reward is an appropriate {\em weighting} of this decrease, with the weighting tilted to favor singleton elements over entire rows. Similarly, each removal results in an increase in the loss; the corresponding cost is the {\em same} weighting of this increase. Each iteration consists of one forward step, and (potentially) several backward steps. We maintain two sets: $\Os \subset [p]\times [r]$ of singleton elements, and $\Ob\subset [r]$ of rows\footnote{We abuse notation by using $\Ob$ to also refer to all elements in these rows. The correct usage is always clear from context.}.
\begin{enumerate}
\item In the {\em forward step}, we find the new object whose inclusion would yield the {\em highest} reward. If this reward is large enough, the object is included in its corresponding matrix (i.e. $S$ or $B$). We also record the value of the reward, and the type of object. If this reward is less than an absolute threshold, the algorithm terminates.
\item In each {\em backward step}, we find the object with the {\em lowest} cost. If this cost is ``low enough", we remove the object. Else we do not. We now explain what ``low enough" means; this is crucial. Say the object with lowest cost is a singleton element, and there are currently $k$ singleton elements in the matrix $S$. Then we remove this element if its cost is smaller, by a fixed fraction $\nu<1$, than the reward obtained when the $k^{th}$ singleton element was added to $S$ (note that, because each iteration has several backward steps, this addition could have happened many forward steps prior to the current iteration). Similarly, if the object was a row, its cost is compared to the corresponding row reward obtained. 
\end{enumerate}

%
%

{\bf Convergence:} It can be seen that in any given iteration, the loss can actually increase! This is because there can be multiple backward steps in the same iteration. To see that the algorithm converges, note that the cost of each backward step is at most the fraction $\nu<1$ of the reward of the  {\em corresponding} forward step: the one  it was compared to when we made the decision to execute the backward step. Thus, this backward and forward step, as a pair, result in a decrease in the loss. Convergence follows from the fact that there is a one to one correspondence between each backward step and its corresponding forward step; there are no backward steps that are ``un-accounted for".


\begin{algorithm*}
\caption{\small Greedy Dirty Model}
\label{Alg:General}
\begin{algorithmic}
\STATE {\bf Input:} Data $\{y^{(1)}, X^{(1)}, \ldots, y^{(r)}, X^{(r)}\}$, Stopping Threshold $\epsilon$, Sharing threshold weight $w\in (1,r)$, Backward Factor $\nu \in (0,1)$
\STATE {\bf Output Variables}: set of singleton elements $\Os\subset [p]\times [r]$, set of rows $\Ob\subset [p]$

\vspace{0.1in}

\STATE{\bf Initialize}: $\Os = \emptyset$, $\Ob = \emptyset$, $\widehat{\beta} = 0$, $k=0$

\vspace{0.1in}

\WHILE[\textit{Forward Step}]{true}

	\vspace{0.1in}

	\STATE Find the best new singleton element and its reward
	\[
	[ \mu_s , ( \hati,\hatj ) ] ~ \leftarrow ~ \max_{(i,j)\notin \Os} \left \{ \L(\widehat{\beta}) \, -\, \min_{\gamma\in \real} \L(\widehat{\beta}+\gamma e_i e_j^T) \, \right \} 
	\]
	\STATE Find the best new row and its reward
	\[
	[ \mu_b , \hatm ] ~ \leftarrow ~ \frac{1}{w} ~ \times ~ \max_{m\notin \Ob} \left \{ \L(\widehat{\beta}) \, -\, \min_{\alpha\in \real^r} \L(\widehat{\beta}+e_m \alpha^T) \, \right \} 
	\]
	\STATE Choose and record the bigger weighted gain $\mu^{(k)} \leftarrow \max(\mu_s,\mu_b)$ 

	\vspace{0.1in}

	\IF[\textit{Gain too small}]{$\mu^{(k)} \leq \epsilon$} 
		\STATE {\bf break} {\em (algorithm stops)}
	\ENDIF 

	\vspace{0.1in}

	\STATE {\bf If} $\mu_b \geq\mu_s$ {\bf then} add row $\Ob \leftarrow \Ob \cup \hatm$, {\bf else} add singleton $\Os \leftarrow \Os \cup (\hati,\hatj)$

	\vspace{0.1in}

 	\STATE Re-estimate on the new support set 
	\[ 
	\widehat{\beta} \leftarrow \arg \min_{\beta:supp(\beta)\subset \Os\cup\Ob} \L(\beta)
 	\]
	\STATE Increment $k\leftarrow k+1$

	\vspace{0.1in}

	\WHILE[\textit{Several backward steps for each forward step}]{true}

		\vspace{0.1in}

		\STATE Find the worst singleton element and its cost
		\[
		[ \nu_s , ( \hati,\hatj ) ] ~ \leftarrow ~ \min_{(i,j) \in \Os} \left \{ \L(\widehat{\beta}-\widehat{\beta}_i^{(j)}e_ie_j^T) \, -\,  \L(\widehat{\beta}) \, \right \} 
		\]
		\STATE Find the worst row and its cost
		\[
		[ \nu_b , \hatm ] ~ \leftarrow ~ \frac{1}{w} ~ \times \min_{m \in \Ob} \left \{  \L(\widehat{\beta}-e_m\widehat{\beta}_m)  \, -\,  \L(\widehat{\beta}) \, \right \} 
		\]

		\vspace{0.1in}

		\IF[\textit{Cost too large}]{$\min(\nu_s,\nu_b) > \nu \epsilon \mu^{(k-1)}$} 
			\STATE {\bf break} {\em (backward steps end)}
		\ENDIF

		\vspace{0.1in}

		\STATE {\bf If} $\nu_b \leq\nu_s$ {\bf then} remove row $\Ob \leftarrow \Ob - \hatm$, {\bf else} remove singleton $\Os \leftarrow \Os - (\hati,\hatj)$		

		\vspace{0.1in}

		\STATE Re-estimate on the new support set 
			\[ 
			\widehat{\beta} \leftarrow \arg \min_{\beta:supp(\beta)\subset \Os\cup\Ob} \L(\beta)
	 	\]

		\vspace{0.1in}
		
		\STATE Decrement $k \leftarrow k-1$

		\vspace{0.1in}

	\ENDWHILE

	\vspace{0.1in}
	
\ENDWHILE
\end{algorithmic}
\end{algorithm*}

\section{Performance Guarantees}
{\em Shared and non-shared features:} Consider the true matrix $\beta^*$, and for a fixed value of integer $d$ define the set of ``shared" features/rows 
$\Omega_b^* ~ := ~ \{i\in [p] \, | \, |supp(\beta^*_i)| > d \}$ that have support more than $d$. In this paper, we overload notation so that $\Omega_b^*$ refers to both the set of rows above (in which case $\Omega_b^*\subset [p]$) and the set of all elements in these rows (in which case $\Omega_b^*\subset [p]\times [r]$); correct interpretation is always clear from context. We can also define support on the non-shared features $\Omega^*_s \subset [p]\times [r]$ as follows $\Omega^*_s ~ := ~ \{(i,j)\in [p]\times [r] \, | \, \beta^*_{ij} \neq 0 ~\text{and $i\notin \Omega^*_b$}\}$ and finally, we define $s^*_j=|\Omega^*_b|+|\{i:(i,j)\in\Omega^*_s\}|$. 

Recall that our method requires a number $w\in (1,r)$ as an input, and outputs two sets --  a set $\Os\subset [p]\times [r]$ of singleton elements, and a set $\Ob\subset [p]$ of rows  -- and an estimated matrix $\widehat{\beta}$ which is supported on $\Os\cup\Ob$. 
Our main analytical result, Theorem \ref{thr:main} below, is a deterministic quantification of conditions (on $X,z,\beta^*$) under which our algorithm with $w\in(d-1,d)$ as input, yields sparsistency -- i.e. recovery of the shared rows $\Ob=\Omega_b^*$, the support on the non-shared rows $\Os = \Omega^*_s$ -- and small error $\|\widehat{\beta} - \beta^*\|_F$. We start with the assumptions and then state the theorem. Corollary \ref{cor:random} covers two popular scenarios with randomness: where the design matrices $X$ are deterministic but the noise vectors $z$ are Gaussian, and the case where both $X$ and $z$ are Gaussian.

{\bf Restricted Eigenvalue Property (REP):} Fix a $j$, and sparsity level $s_j$. We say the matrix $Q^{(j)} :=X^{(j)}X^{(j)T}$ satisfies $REP(s_j)$ with constants $C_{\min}$ and $\rho\geq 1$, if for all $j$, and all $s_j$-sparse vectors $\delta\in\real^p$, we have that
\begin{equation}
\label{Eq:REP}
\begin{aligned}
C_{\min}\|\delta\|_2\leq\|Q^{(j)}\delta\|_2\leq\rho C_{\min}\|\delta\|_2\qquad\forall \|\delta\|_0\leq s_j\\
\end{aligned}
\end{equation}
In our results, we assume (by taking the maximum/minimum over all tasks) that $C_{\min}$ and $\rho$ are the same for all tasks. Note that the level of sparsity $s_j$ will still be different for different $j$.

{\bf Gradient of the loss function:} If there is no noise, i.e. $z^{(j)} = 0$ for all $j$, then  $\beta^*$ is the optimal point of the loss function and the loss function has zero gradient there, i.e. $\nabla\Loss(\beta^*) = 0$. However, for any $j$ if $z^{(j)} \neq 0$, the corresponding gradient $\nabla^{(j)} := -X^{(j)T}\left(y^{(j)}-X^{(j)}\beta^{*(j)}\right)$ will not be zero either. We define $\lambda$ to be an upper bound on the infinity norm of this gradient, i.e. $\lambda := \max_j \|\nabla^{(j)}\|_\infty$.

{\bf Minimum non-zero element:} Elements of $\beta$ with very small magnitude are hard to distinguish from 0, so we need to specify a lower bound on the magnitude of elements in the support of $\beta$ we want to recover. Towards this end, for a given $d$, suppose $\bar{\beta}^{*}_{m,d}$ is the magnitude of the  $d^{th}$ largest entry (by magnitude) entry in row $m$ of $\beta^{*}$. and $\bar{\beta}^{*}_{d}=\min_{m\in\Omega^*_b}\,\bar{\beta}^{*}_{m,d}$.  Finally, let $\beta^*_{\min}=\min\left \{ \min_{(i,j)\in\Omega_s}\beta^{(j)*}_i\,,\,\min_{m\in\Omega^*_b}\,\bar{\beta}^{*}_{m,d} \right \}$. Note that a lower bound on $\beta_{\min}$ implies that there at least $d$ elements whose magnitude is above that bound in every row of $\Omega^*_b$, and also that every element in $\Omega^*_s$ is above that bound.

\def\RSMult{\eta}
\begin{theorem}[Sparsistency] \label{thr:main}
Suppose the algorithm is run with $\epsilon,w$ and $\nu$. Let $d$ be such that $d-1< w< d$, and for this $d$ let shared rows $\Omega^*_b$, non-shared features $\Omega^*_s$, sparsity levels $\{s^*_j\}$, and minimum element $\beta^*_{\min}$ as above. Suppose also that $\beta^*_{\min} > 4\rho\sqrt{\epsilon}/C_{\min}$, and for each $j$ we have that $REP\left(\RSMult \, s^*_j\right)$ holds with constants $C_{\min}$ and $\rho$, and that {\small $\RSMult \geq\,2 + 4r\rho^4(\rho^4-\rho^2+2)/(w\,\nu)$}. Then, if we run Algorithm~\ref{Alg:General} with stopping threshold $\epsilon \ge \frac{4 \rho^2 \RSMult r^2 s^* \lambda^2}{w \nu C_{\min}^2}$ , the output $\widehat{\beta}$ with shared support $\Ob$ and individual support $\Os$ satisfies: 
\begin{itemize}
	\item[(a)] {\bf Error Bound:} $\|\widehat{\beta} - \beta^*\|_F \le \frac{\sqrt{rs^*}}{C_{\min}}\left(\frac{\lambda\sqrt{\RSMult}}{C_{\min}} + 2\rho\sqrt{\epsilon}\right).$
	\item[(b)] {\bf Support Recovery:} $\Ob=\Omega^*_b$ and $\Os=\Omega^*_s$.\\
\end{itemize}	
\end{theorem}

{\bf Remark 1.} The noiseless case $z=0$ corresponds to $\lambda=0$, in which case the algorithm can be run with $\epsilon =0$. As can be seen, this yields exact recovery, i.e. $\widehat{\beta} = \beta^*$.

{\bf Remark 2.} The smaller the value of the backward factor $\nu$, the faster the algorithm is likely to converge as there are likely to be fewer backward steps. However, smaller $\nu$ results in larger values of $\RSMult$ and $\beta_{\min}$ that we need for success; thus an algorithm with smaller $\nu$ is likely to work on a smaller range of problems: a trade-off between statistical and computational complexity.

{\bf Remark 3.} Note that all the rows in $\Os$ has less than $d$ elements. To see this, suppose in contrary that there exist a row $m$ in $\Os$ that has more than or equal to $d$ non-zeros. Since in the algorithm these single elements should compete with $\frac{1}{w}$ times the improvement of the row, and $d-1<w<d$, the row $m$ will be chosen for $\Ob$ before those $d$ entries are chosen for $\Os$. Once the row $m$ goes to $\Ob$, since we optimize for each entry on the row separately, it is impossible that any other single element on that row goes to $\Os$. Hence, rows of $\Os$ have less than $d$ entries and can be distinguished from the rows of $\Ob$.

{\bf Remark 4.} Some recent results \cite{JJR11,TWD11} study greedy algorithms in a general ``atomic" framework. While our setting could be made to fall into this general framework, the resulting algorithm would be different, and the performance guarantees would be weaker. These results require $REP(\eta\sum s^*_j)$ for ``each" and ``all" task, which is order-wise (by an order of $r$) worse than our assumption $REP(\eta s^*_j)$ for task $j$. To get this result, we leverage the fact that our loss function is separable with respect to tasks and hence, we do the analysis on a per-task basis.

\begin{corollary} \label{cor:random}
For sample complexity $n_j\geq c_1\,s_j\log(rp)$, with probability at least $1-c_2\exp(-c_3n)$ for some constants $c_1-c_4$, we have

\begin{itemize}
\item [(C1)] Under the assumptions of the Theorem~\ref{thr:main}, if $z$ is $\mathcal{N}(0,\sigma)$, then the result holds for $\lambda=c_4\sqrt{\frac{\log(rp)}{n}}$ for some constant $c_4$.\\

\item [(C2)] If $X^{(j)}$ is $\mathcal{N}(0,\Sigma^{(j)})$ and REP assumption in Theorem~\ref{thr:main} holds for $Q^{(j)}=\Sigma^{(j)}\Sigma^{(j)T}$, then the result holds.
\end{itemize}
\end{corollary}

\begin{proof} [Proof of Theorem~\ref{thr:main}]
Let $\widehat{s}_j=|\Os^{(j)}\cup\Ob\cup\Omega^{*(j)}_s\cup\Omega^*_b|$ be the size of the support of the estimated $j^{th}$ task union with the support of the true $j^{th}$ task. Inspired by \cite{JJR11}, our proof is based upon the following two lemmas:

\begin{lemma}
If $REP(\widehat{s}_j)$ holds, then
\begin{itemize}
\item [(i)] {\small$\left\|\widehat{\beta}^{(j)} - \beta^{*(j)}\right\|_2 \leq \frac{1}{C_{\min}} \left(\frac{\noiseLevel\, \sqrt{\widehat{s}_j}}{C_{\min}} +2\rho\sqrt{|(\Omega^*_b\cup\Omega^{*(j)}_s)-(\Ob\cup\Os^{(j)})|\epsilon}\right)$}.
\item [(ii)] {\small$\left\|\widehat{\beta}^{(j)}_{(\Os^{(j)} \cup \Ob) - (\Omega^{*(j)}_s\cup \Omega^*_b)}\right\|_2 \ge \frac{\sqrt{w\nu\epsilon}}{\sqrt{r}\,\rho C_{\min}} \sqrt{|(\Os^{(j)} \cup \Ob) - (\Omega^{*(j)}_s\cup \Omega^*_b)|}$}.
\end{itemize}
\label{lem:LemErrorBound}
\end{lemma}
\begin{lemma}
If $\epsilon$ is chosen properly (see appendix for the exact expression), then $k$ never exceeds $(\eta-1)s^*_j$, and hence, $\widehat{s}_j\leq k+s^*_j\leq\eta s^*_j$.
\label{lem:LemStoppingSize}
\end{lemma}
Part (i) and (ii) of Lemma~\ref{lem:LemErrorBound} are consequences of the fact that when algorithm stops the forward step and previous backward step fail to go through, respectively. To ensure the assumption of Lemma~\ref{lem:LemErrorBound} holds, we need the Lemma~\ref{lem:LemStoppingSize} that bounds $\widehat{s}_j$. The proof can be completed as below.

\begin{list}{\labelitemi}{\leftmargin=1em}
\item[(a)] The result follows directly from part (i) of Lemma~\ref{lem:LemErrorBound} noting that $s\le \RSMult s^*$ by Lemma~\ref{lem:LemStoppingSize}.

\item [(b)] Considering {\bf Remark 3}, we only need to show that $(\Omega_b^* \cup \Omega_s^{*(j)})- (\Ob \cup \Os^{(j)})=(\Ob \cup \Os^{(j)})-(\Omega_b^* \cup \Omega_s^{*(j)})=\emptyset$. For any $\tau\in\real$, we have
\small\begin{align*}
	(\beta^*_{\min})^2 \left|(\Omega_b^* \cup \Omega_s^{*(j)})- (\Ob \cup \Os^{(j)})\right| &= (\beta^{*}_{\min})^2 \left|\{(i,j) \in (\Omega_b^* \cup \Omega_s^{*(j)})- (\Ob \cup \Os^{(j)}) : |\beta_i^{*(j)}| \geq \beta^{*}_{\min}\}\right|\\ 
			&\le \|\beta^*_{(\Omega_b^* \cup \Omega_s^{*(j)})- (\Ob \cup \Os^{(j)})}\|_2^2 \,\le \|\beta^* - \widehat{\beta}\|_2^2\\
			&\le \frac{2 \RSMult r s^*\noiseLevel^2}{C_{\min}^4} + \frac{8 \rho^2 \epsilon}{C_{\min}^2}|(\Omega_b^* \cup \Omega_s^{*(j)})- (\Ob \cup \Os^{(j)})|,
\end{align*}\normalsize
where the last inequality follows from part (a) and the inequality $(a+b)^2 \leq 2a^2+2b^2$. Now, dividing both sides by $\beta^{*2}_{\min}/2$ we get

\vspace{-0.4cm}
\small\begin{align*}
	2 |(\Omega_b^* \cup \Omega_s^{*(j)})- (\Ob \cup \Os^{(j)})| &\le \frac{4\RSMult r s^* \noiseLevel^2}{C_{\min}^4 (\beta^*_{\min})^2} + \frac{16\rho^2\epsilon}{C_{\min}^2(\beta^*_{\min})^2}|(\Omega_b^* \cup \Omega_s^{*(j)})- (\Ob \cup \Os^{(j)})|\\ &\leq \frac{1}{2} + |(\Omega_b^* \cup \Omega_s^{*(j)})- (\Ob \cup \Os^{(j)})|.
\end{align*}\normalsize
The inequality follows from the assumption on $\epsilon$ and $\beta^*_{\min}$ implying $|(\Omega_b^* \cup \Omega_s^{*(j)})- (\Ob \cup \Os^{(j)})| =0$. To show the converse, from part (ii) of Lemma~\ref{lem:LemErrorBound}, we have

\vspace{-0.4cm}
\small\begin{equation}
\begin{aligned}
|(\Ob \cup \Os^{(j)})-(\Omega_b^* \cup \Omega_s^{*(j)})| &\leq \frac{r\rho^2 C_{\min}^2}{w\nu\epsilon} \|\widehat{\beta}^{(j)}_{(\Ob \cup \Os^{(j)})-(\Omega_b^* \cup \Omega_s^{*(j)})}\|_2^2\leq \frac{r\rho^2 C_{\min}^2}{w\nu\epsilon} \left\|\widehat{\beta}^{(j)}-\beta^{*(j)}\right\|_2^2\\
&\leq \frac{r\rho^2 C_{\min}^2}{w\nu\epsilon} \frac{2 \RSMult r s^* \lambda^2}{C_{\min}^4}\leq 1/2
\end{aligned}
\nonumber
\end{equation}\normalsize
due to the setting of  the stopping threshold $\epsilon$. This implies that $|(\Ob \cup \Os^{(j)})-(\Omega_b^* \cup \Omega_s^{*(j)})|=0$ and concludes the proof of the theorem.
\end{list}
\vspace{-0.8cm}
\end{proof}

\section{Experimental Results}\label{SecExper}
\begin{figure}[t]
\centering
\subfigure[\small Little overlap: $\kappa=0.3$]{
\includegraphics[width=0.46\linewidth]{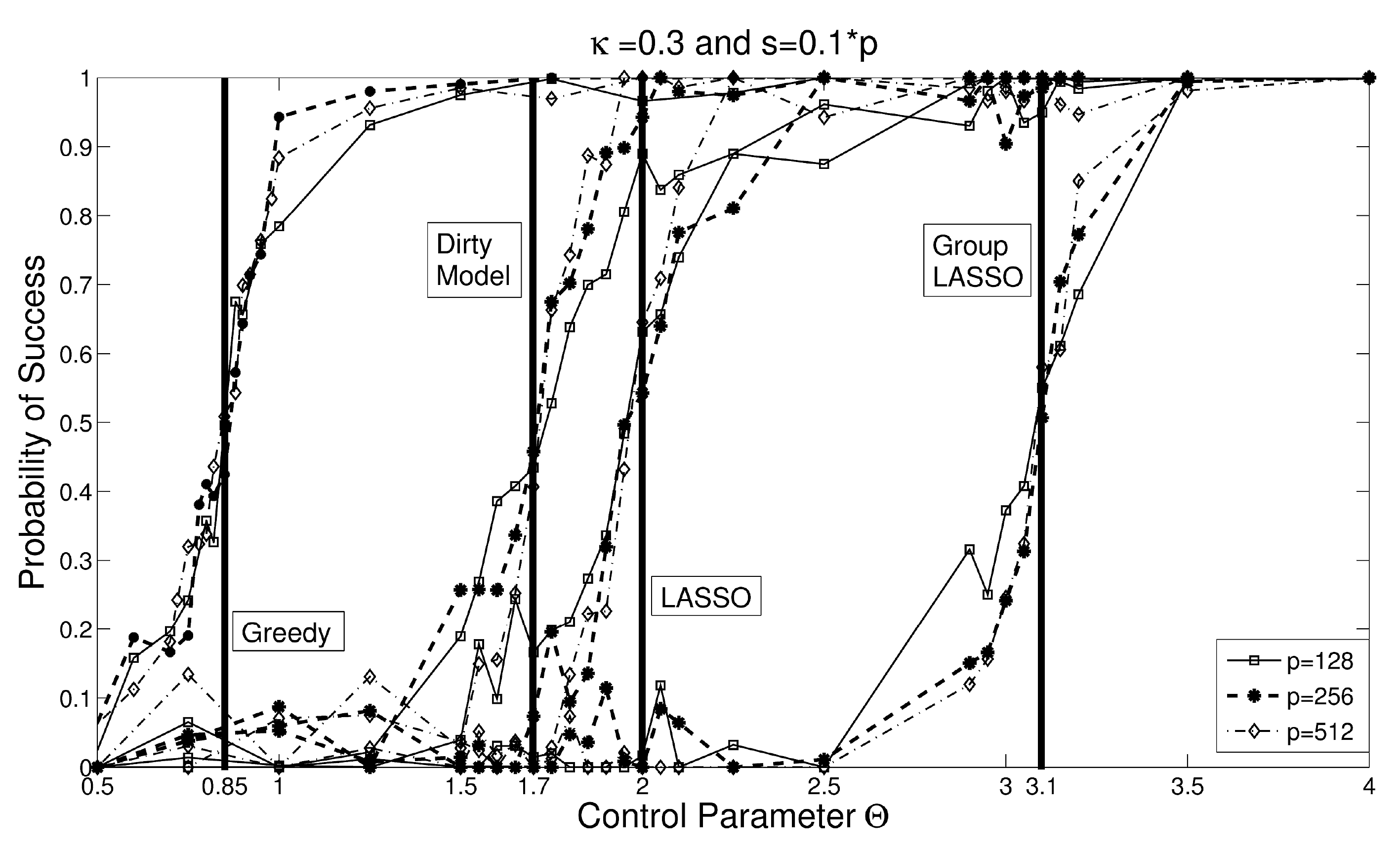}
\label{fig:fig1a}
}
\subfigure[\small Moderate overlap: $\kappa=2/3$]{
\includegraphics[width=0.46\linewidth]{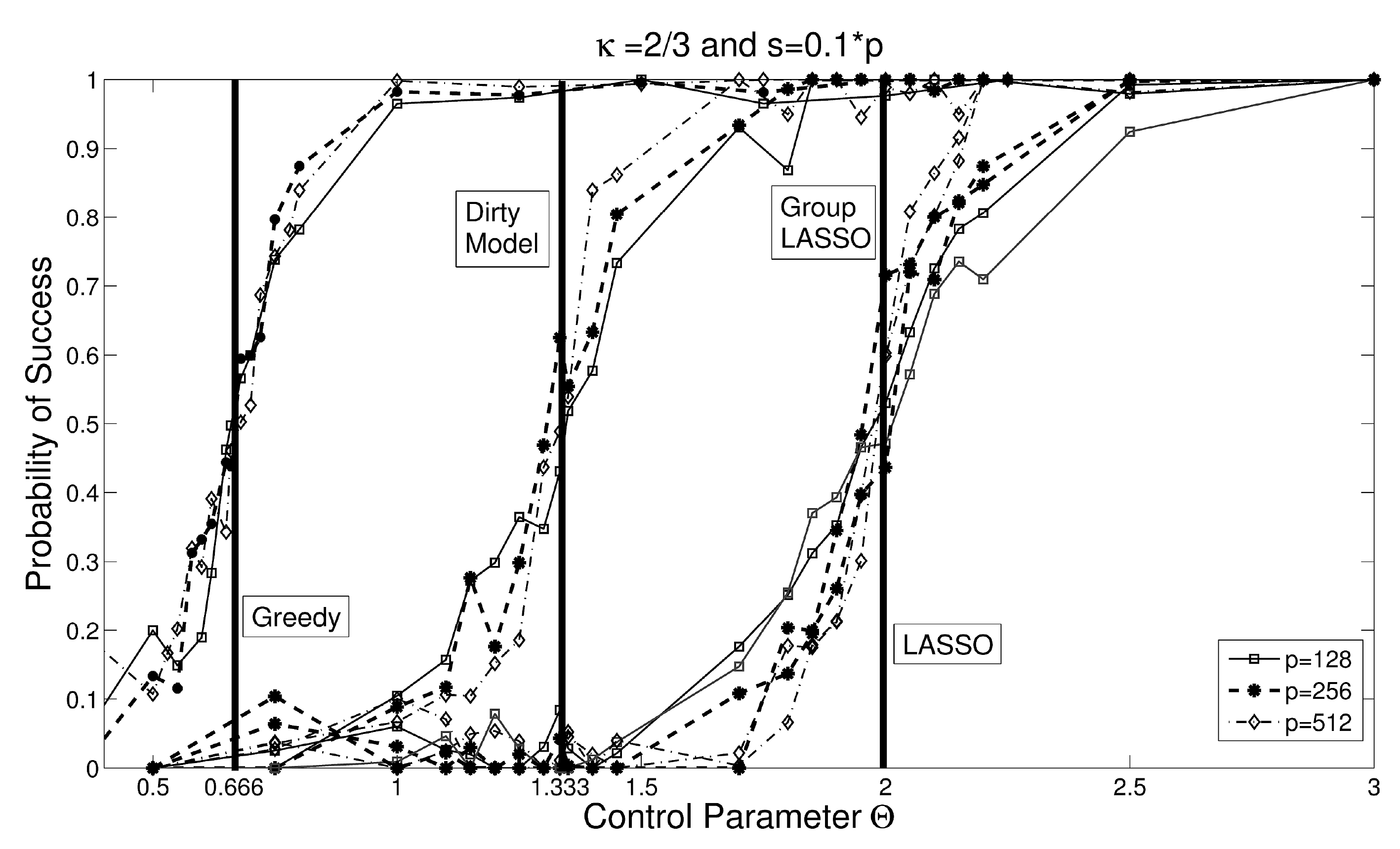}
\label{fig:fig1b}
}
\subfigure[\small High overlap: $\kappa=0.8$]{
\includegraphics[width=0.46\linewidth]{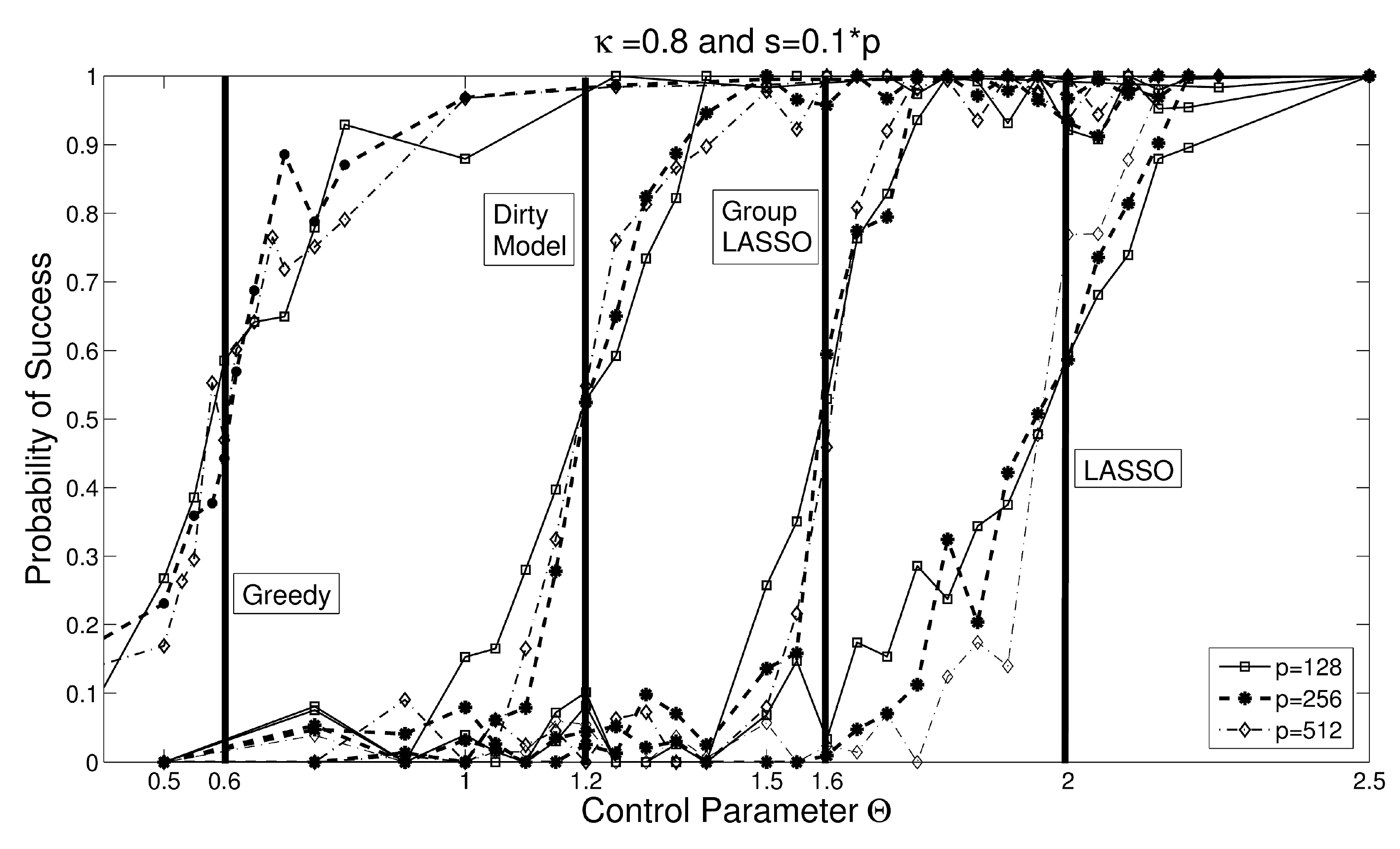}
\label{fig:fig1c}
}
\label{fig:fig1}
\caption{\small Probability of success in recovering the exact sign support using greedy algorithm, dirty model, Lasso and group LASSO ($\ell_1/\ell_\infty$). For a 2-task problem, the probability of success for different values of feature-overlap fraction $\kappa$ is plotted. Here, we let $s=p/10$ and the values of the parameter and design matrices are i.i.d standard Gaussians and $\sigma=0.1$. Greedy method outperforms all methods in the sample complexity required for sign support recovery.}
\end{figure}


\subsection{Synthetic Data}
To have a common ground for comparison, we run the same experiment used for the comparison of LASSO, group LASSO and dirty model in \cite{NWJoint,JRSR10}. Consider the case where we have $r=2$ tasks each with the support size of $s=p/10$ and suppose these two tasks share a $\kappa$ portion of their supports. The location of non-zero entries are chosen uniformly at random and values of $\beta^*_1$ and $\beta^*_2$ are chosen to be standard Gaussian realizations. Each row of he matrices $X^{(1)}$ and $X^{(2)}$ is distributed as $\mathcal{N}(0,I)$ and each entry of the noise vectors $w_1$ and $w_2$ is a zero-mean Gaussian draw with variance $0.1$. We run the experiment for problem sizes $p\in{128,256,512}$ and for support overlap levels $\kappa\in{0.3,2/3,0.8}$.

We use cross-validation to find the best values of regularizer coefficients. To do so, we choose $\epsilon = c\,\frac{s\log(p)}{n}$, where $c\in[10^{-4},10]$, and $w\in[1,2]$. Notice that this search region is motivated by the requirements of our theorem and can be substantially smaller than the region needs to be searched for $\epsilon$ and $w$ if they are independent. Interestingly, for small number of samples $n$, the ratio $w$ tends to be close to $1$, where for large number of samples, the ratio tends to be close to $2$. We suspect this phenomenon is due to the lack of curvature around the optimal point when we have few samples. The greedy algorithm is more stable if it picks a row as opposed to a single coordinate, even if the improvement of the entire row is comparable to the improvement of a single coordinate.

To compare different methods under this regime, we define a rescaled version of sample size $n$, aka control parameter $\Theta=\frac{n}{s\log\left(p-(2-\kappa)s\right)}.$
For different values of $\kappa$, we plot the probability of success, obtained by averaging over 100 problems, versus the control parameter $\Theta$ in Fig.\ref{fig:fig1}. It can be seen that the greedy method outperforms, i.e., requires less number of samples, to recover the exact sign support of $\beta^*$. 

This result matches the known theoretical guarantees. It is well-known that LASSO has a sharp transition at $\Theta\approx 2$ \cite{WainwrightLasso}\footnote[1]{The exact expression is $\frac{n}{s\log(p)}=2$. Here, we ignore the term $(2-\kappa)s$ comparing to $p$.}, group LASSO ($\ell_1/\ell_\infty$ regularizer) has a sharp transition at $\Theta=4-3\kappa$ \cite{NWJoint} and dirty model has a sharp transition at $\Theta=2-\kappa$ \cite{JRSR10}. Although we do not have a theoretical result, these experiments suggest the following conjecture:

\begin{conjecture}
For two-task problem with $C_{\min}=\rho=1$ and Gaussian designs, the greedy algorithm has a sharp transition at $\Theta = 1 - \frac{\kappa}{2}$.
\end{conjecture}

\begin{figure}[t]
\centering
\includegraphics[width=0.6\linewidth]{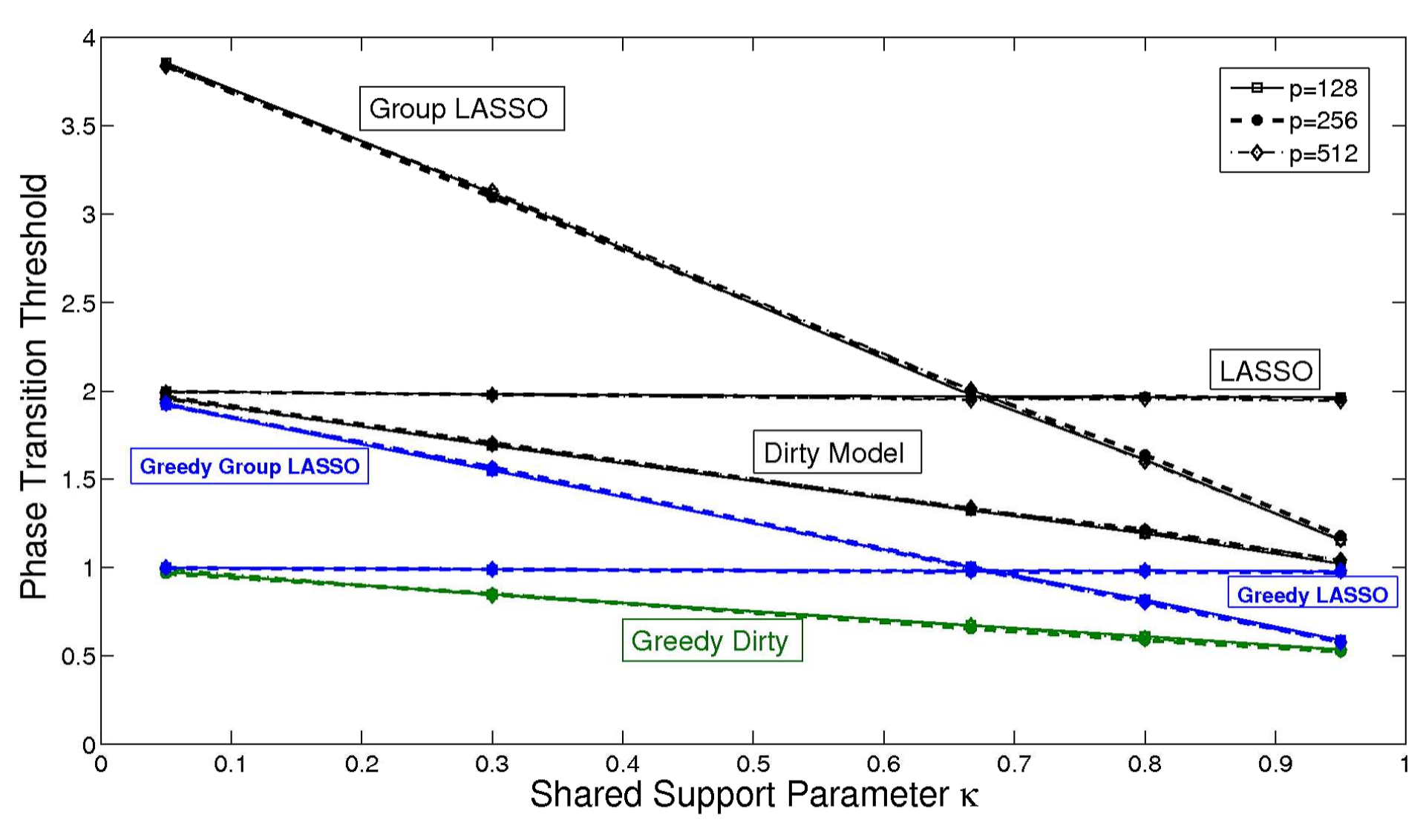}
\caption{\small Phase transition threshold versus the parameter $\kappa$ in a 2-task problem for greedy algorithm, dirty model, LASSO and group LASSO ($\ell_1/\ell_\infty$ regularizer). The y-axis is $\Theta=\frac{n}{s\log(p-(2-\kappa)s)}$. Here, we let $s=p/10$ and the values of the parameter and design matrices are i.i.d standard Gaussians and $\sigma=0.1$. The greedy algorithm shows substantial improvement in sample complexity over the other methods.}
\label{fig:fig2}
\end{figure}

\vspace{-0.4cm}
To investigate our conjecture, we plot the sharp transition thresholds for different methods versus different values of $\kappa\in\{0.05,0.3,2/3,0.8,0.95\}$ for problem sizes $p\in\{128,256,512\}$. Fig~\ref{fig:fig2} shows that the sharp transition threshold for greedy algorithm follows our conjecture with a good precision. Although, theoretical guarantee for such a tight threshold remains open.

\subsection{Handwritten Digits Dataset}


\begin{table*}
\centering
\scriptsize
\begin{tabular}{|c||c|c|c|c|c|}
\hline
\textbf{n}&&\textbf{Greedy}&\textbf{Dirty Model} &\textbf{Group LASSO}& \textbf{LASSO}\\
\hline\hline
10 & Average Classification Error & 6.5\% & 8.6\% & 9.9\% & 10.8\%\\ 
& Variance of Error & 0.4\% & 0.53\% & 0.64\% & 0.51\%\\
& Average Row Support Size & 180 & 171 & 170 & 123\\
& Average Support Size & 1072 & 1651 & 1700 & 539\\
\hline
20 & Average Classification Error & 2.1\% & 3.0\% & 3.5\% & 4.1\%\\ 
& Variance of Error & 0.44\% & 0.56\% & 0.62\% & 0.68\%\\
& Average Row Support Size & 185 & 226 & 217 & 173\\
& Average Support Size & 1120 & 2118 & 2165 & 821\\
\hline
40 & Average Classification Error & 1.4\% & 2.2\% & 3.2\% & 2.8\%\\ 
& Variance of Error & 0.48\% & 0.57\% & 0.68\% & 0.85\%\\
& Average Row Support Size & 194 & 299 & 368 & 354\\
& Average Support Size & 1432 & 2761 & 3669 & 2053\\
\hline
\end{tabular}
\caption{\small Handwriting Classification Results for greedy algorithm, dirty model, group LASSO and LASSO. The greedy method provides much better classification errors with simpler models. The greedy model selection is more consistent as the number of samples increases.}
\label{tab:tab1}
\end{table*}

We use the handwritten digit dataset \cite{DSet} that is used by a number of papers \cite{PERTHE,HENIY,JRSR10} as a reliable dataset for optical handwritten digit recognition algorithms. The dataset contains $p=649$ features of handwritten numerals 0-9 ($r=10$ tasks) extracted from a collection of Dutch utility maps. The dataset provides $200$ samples of each digit written by different people. We take $n/10$ samples from each digit and combine them to a big matrix $X\in\real^{n\times p}$, i.e., we set $X^{(i)}=X$ for all $i\in\{1,\ldots,10\}$. We construct the response vectors $y_i$ to be $1$ if the corresponding row in $X$ is an instance of $i^{th}$ digit and zero otherwise. Clearly, $y_i$'s will have a disjoint support sets. We run all four algorithms on this data and report the results.

Table~\ref{tab:tab1} shows the results of our analysis for different sizes of the training set $n$. We measure the classification error for each digit to get the 10-vector of errors. Then, we find the average error and the variance of the error vector to show how the error is distributed over all tasks. Again, in all methods, parameters are chosen via cross-validation. It can be seen that the greedy method provides a more consistent model selection as the model complexity does not change too much as the number of samples increases while the classification error decreases substantially. In all cases, we get $\% 25-\% 30$ improvement in classification error.


\bibliographystyle{plainnat}
\bibliography{DirtyGreedy}

\appendix
\section{Auxiliary Lemmas for Theorem~1}


Note that when the algorithm terminates, the forward step fails to go through. This entails that 
\begin{equation}
\begin{aligned}
\Loss(\widehat{\beta}) - \inf_{m \notin \Ob,\alpha \in \real^r} \Loss(\widehat{\beta} +  e_m\,\alpha^T) &< w\epsilon\\
\Loss(\widehat{\beta}) - \inf_{(i,j) \notin \Os,\gamma \in \real} \Loss(\widehat{\beta} + \gamma e_ie_j^T) &< \epsilon.
\end{aligned}
\nonumber
\end{equation}
Since our loss function is separable with respect to tasks, i.e., $\Loss(\widehat{\beta})=\sum_j\Loss(\widehat{\beta}^{(j)}e_j^T)$, for a fixed task $j$, we can rewrite the second inequality as
\begin{equation}
\begin{aligned}
\Loss(\widehat{\beta}^{(j)}e_j^T) - \inf_{(i,j) \notin \Os,\gamma \in \real} \Loss(\widehat{\beta}^{(j)}e_j^T + \gamma e_ie_j^T) &< \epsilon.
\end{aligned}
\nonumber
\end{equation}

The next lemma shows that this has the consequence of upper bounding the deviation in loss between the estimated parameters $\widehat{\beta}$ and the true parameters $\beta^*$.

\begin{lemma}[Stopping Forward Step] When the algorithm stops with parameter $\widehat{\beta}$, we have
\label{LemForwardStep}
\begin{align}
	\left|\Loss\left(\widehat{\beta}^{(j)}e_j^T\right) - \Loss\left(\beta^{*(j)}e_j^T\right)\right| < 2\rho C_{\min} \sqrt{|(\Omega^{*(j)}_s\cup\Omega^*_b)-(\Os^{(j)}\cup\Ob)|\epsilon} \; \left\|\widehat{\beta}^{(j)} - \beta^{*(j)}\right\|_2. 
\end{align}
\end{lemma}

\begin{proof}
Let $\widehat{\delpar} = \beta^* - \widehat{\beta}$. For any $\eta\in\real$, we have
\begin{equation}
\begin{aligned}	
-|(\Omega^{*(j)}_s\cup\Omega^*_b)&-(\Os^{(j)}\cup\Ob)|\epsilon\\  &< \sum_{(i,j) \in (\Omega^{*(j)}_s\cup\Omega^*_b)-(\Os^{(j)}\cup\Ob)}\!\!\!\!\! \Big(\Loss\left(\widehat{\beta}^{(j)}e_j^T\! + \eta \widehat{\delpar}_{i}^{(j)} e_ie_j^T\right)\!\! - \Loss\left(\widehat{\beta}^{(j)}e_j^T\right)\!\!\Big)\\ &\leq \eta\tr{\nabla\L(\widehat{\beta}^{(j)}e_j^T)}{\widehat{\delpar}_{(\Omega^{*(j)}_s\cup\Omega^*_b)-(\Os^{(j)}\cup\Ob)}^{(j)}} +\eta^2\rho^2C_{\min}^2\|\widehat{\delpar}_{(\Omega^{*(j)}_s\cup\Omega^*_b)-(\Os^{(j)}\cup\Ob)}^{(j)}\|_2^2\\ &\leq \eta \Big(\Loss\left(\beta^{*(j)}e_j^T\right) - \Loss\left(\widehat{\beta}^{(j)}e_j^T\right)\Big) + \eta^2 \rho^2C_{\min}^2\|\widehat{\delpar}^{(j)}\|_2^2.
\end{aligned}
\label{eq:epsbound}
\end{equation}
Here, we use the fact that $\nabla\Loss(\widehat{\beta}^{(j)}e_j^T)$ is zero on the support of $\widehat{\beta}^{(j)}$. Optimizing the RHS over $\eta$, we obtain
\begin{equation}
\begin{aligned}
-|(\Omega^{*(j)}_s\cup\Omega^*_b)-(\Os^{(j)}\cup\Ob)|\epsilon &< -\, \frac{\Big(\Loss(\beta^{*(j)}e_j^T) - \Loss\left(\widehat{\beta}^{(j)}e_j^T\right)\Big)^2}{4\, \rho^2 C_{\min}^2\, \|\widehat{\delpar}^{(j)}\|_2^2},
\end{aligned}
\nonumber
\end{equation}
whence the lemma follows.\\
\end{proof}

\begin{lemma}[Stopping Error Bound] When the algorithm stops with parameter $\widehat{\beta}$, we have
\label{LemErrorBound}
\begin{equation}
\|\widehat{\beta}^{(j)} - \beta^{*(j)}\|_2 \leq \frac{1}{C_{\min}} \left(\frac{\noiseLevel}{C_{\min}} \sqrt{|(\Omega^{*(j)}_s\cup\Omega^*_b)\cup(\Os^{(j)}\cup\Ob)|} +2\rho\sqrt{|(\Omega^{*(j)}_s\cup\Omega^*_b)-(\Os^{(j)}\cup\Ob)|\epsilon}\right).
\nonumber
\end{equation}	
\end{lemma}
\begin{proof}
For $\delpar\in\real^p$, let 
\begin{equation}
	G(\delpar) = \Loss\left(\beta^{*(j)}e_j^T +  \delpar e_j^T\right) - \Loss\left(\beta^{*(j)}e_j^T\right) - 2\rho C_{\min} \sqrt{|(\Omega^{*(j)}_s\cup\Omega^*_b)-(\Os^{(j)}\cup\Ob)| \epsilon}\; \|\delpar\|_2.
\nonumber
\end{equation}
It can be seen that $G(0) = 0$, and from the previous lemma, $G(\delparHat^{(j)}) \le 0$. Further, $G(\delpar)$ is sub-homogeneous (over a limited range): $G(t \delpar) \le t G(\delpar)$ for $t \in [0,1]$ by basic properties of the convex function. Thus, for a carefully chosen $r > 0$, if we show that $G(\delpar) > 0$ for all $\delpar \in \{\delpar : \|\delpar\|_{2} = r,\;\|\delpar\|_0 \le \widehat{s}_j\}$, where, $\widehat{s}_j=|(\Omega^{*(j)}_s\cup\Omega^*_b)\cup(\Os^{(j)}\cup\Ob)|$ is as defined in the proof of the theorem, then, it follows that $\|\delparHat^{(j)}\|_2 \le r$.
If not, then there would exist some $t \in [0,1)$ such that $\|t \delparHat^{(j)}\| = r$, whence we would arrive at the contradiction
\begin{align*}
	0 < G(t \delparHat^{(j)}) \le t G(\delparHat^{(j)}) \le 0.
\end{align*}

Thus, it remains to show that $G(\delpar) > 0$ for all $\delpar \in \{\delpar : \|\delpar\|_{2} = r,\;\|\delpar\|_0 \le \widehat{s}_j\}$. By restricted strong convexity property of $\Loss(\cdot)$, we have
\begin{equation}
\begin{aligned}
\Loss(\beta^{*(j)}e_j^T + \delpar e_j^T) - \Loss(\beta^{*(j)}e_j^T) &= \sum_{j=1}^r\left(\Loss(\beta^* + \delpar e_j^T) - \Loss(\beta^*)\right)\\ &\ge \tr{\grad \Loss(\beta^*)}{\delpar} + C_{\min}^2 \left\|\delpar\right\|_2^2.
\end{aligned}
\nonumber
\end{equation}
We can establish
\begin{align*}
	\tr{\grad \Loss(\beta^*)}{\delpar} &\ge  - \left|\tr{\grad \Loss(\beta^*)}{\delpar}\right|\\
											&\ge - \left\|\grad \Loss(\beta^*)\right\|_\infty \left\|\delpar\right\|_1 = -\noiseLevel \left\|\delpar\right\|_1,
\end{align*}
and hence,
\begin{align*}
G(\delpar) &\ge - \noiseLevel \|\delpar\|_1 + C_{\min}^2 \|\delpar\|_2^2 - 2\rho C_{\min}\sqrt{ |(\Omega^{*(j)}_s\cup\Omega^*_b)-(\Os^{(j)}\cup\Ob)| \epsilon } \|\delpar\|_2\\
						&>  \Big(- \noiseLevel \sqrt{|(\Omega^{*(j)}_s\cup\Omega^*_b)\cup(\Os^{(j)}\cup\Ob)|} + C_{\min}^2 \|\delpar\|_2\\ &\qquad\qquad\qquad\qquad\qquad\qquad- 2\rho C_{\min}\sqrt{ |(\Omega^{*(j)}_s\cup\Omega^*_b)-(\Os^{(j)}\cup\Ob)| \epsilon }  \Big) \|\delpar\|_2 > 0,
\end{align*}
if $\|\delpar\|_2 = r$ for $$r = \frac{1}{C_{\min}} \left(\frac{\noiseLevel}{C_{\min}} \sqrt{|(\Omega^{*(j)}_s\cup\Omega^*_b)\cup(\Os^{(j)}\cup\Ob)|} + 2\rho \sqrt{|(\Omega^{*(j)}_s\cup\Omega^*_b)-(\Os^{(j)}\cup\Ob)| \epsilon}\right).$$

This concludes the proof of the lemma.\\
\end{proof}

Next, we note that when the algorithm terminates, the backward step with the current parameters has failed to go through. This entails that 
\begin{equation}
\begin{aligned}
\inf_{m \in \Ob} \Loss(\widehat{\beta} -  e_m\widehat{\beta}_m) - \Loss(\widehat{\beta})\quad &> \nu\,w\epsilon\\
\inf_{(i,j) \in \Os} \Loss(\widehat{\beta} - \widehat{\beta}^{(j)}_i e_ie_j^T) - \Loss(\widehat{\beta}) &> \nu\epsilon.
\end{aligned}
\end{equation}
The next lemma shows the consequence of this bound.

\begin{lemma}[Stopping Backward Step]
\label{LemBackwardStep}
When the algorithm stops with parameter $\widehat{\beta}$, we have
\begin{equation}\nonumber
\left\|\widehat{\beta}_{(\Os^{(j)} \cup \Ob) - (\Omega^{*(j)}_s\cup \Omega^*_b)}^{(j)}\right\|_2^2 \ge \frac{w\nu\epsilon}{r\rho^2 C_{\min}^2} |(\Os^{(j)} \cup \Ob) - (\Omega^{*(j)}_s\cup \Omega^*_b)|.
\end{equation}
\end{lemma}

\begin{proof}
We have
\begin{equation}
\begin{aligned}
&|(\Os^{(j)} \cup \Ob) - (\Omega^{*(j)}_s\cup \Omega^*_b)|\frac{w}{r}\nu\epsilon\\ 
&\qquad\qquad=|\Os^{(j)} - (\Omega^{*(j)}_s\cup \Omega^*_b\cup\Ob)|\frac{w}{r}\nu\epsilon+\frac{1}{r}|\Ob - (\Omega^{*(j)}_s\cup \Omega^*_b)|\nu\,w\epsilon\\ 
&\qquad\qquad\le \!\!\!\!\!\!\sum_{(i,j) \in \Os^{(j)} - (\Omega^{*(j)}_s\cup \Omega^*_b\cup\Ob)} \left(\Loss(\widehat{\beta} - \widehat{\beta}_i^{(j)} e_ie_j^T) - \Loss(\widehat{\beta})\right) + \frac{1}{r}\!\!\!\!\sum_{m \in \Ob - (\Omega^{*(j)}_s\cup \Omega^*_b)} \left(\Loss(\widehat{\beta} - e_m\widehat{\beta}_m) - \Loss(\widehat{\beta})\right)\\
&\qquad\qquad\le \underbrace{\tr{\grad \Loss(\widehat{\beta})}{\widehat{\beta}_{\Os^{(j)} - (\Omega^{*(j)}_s\cup \Omega^*_b\cup\Ob)}}}_{0} + \rho^2 C_{\min}^2 \|\widehat{\beta}_{\Os^{(j)} - (\Omega^{*(j)}_s\cup \Omega^*_b\cup\Ob)}^{(j)}\|_2^2\\ &\qquad\qquad\qquad\qquad\qquad\qquad\qquad+ \underbrace{\tr{\grad \Loss(\widehat{\beta})}{\widehat{\beta}_{\Ob - (\Omega^{*(j)}_s\cup \Omega^*_b)}}}_{0} + \rho^2 C_{\min}^2 \|\widehat{\beta}_{\Ob - (\Omega^{*(j)}_s\cup \Omega^*_b)}^{(j)}\|_2^2\\
&\qquad\qquad= \rho^2C_{\min}^2 \left\|\widehat{\beta}_{(\Os^{(j)} \cup \Ob) - (\Omega^{*(j)}_s\cup \Omega^*_b)}^{(j)}\right\|_2^2,
\end{aligned}
\label{eq:backeps}
\end{equation}
where, the second inequality uses the fact that $[\grad \Loss(\widehat{\beta})]_{\Os^{(j)}\cup\Ob} = 0$.
\end{proof}


\section{Lemmas on the Stopping Size}

\begin{lemma}
\label{LemStoppingSize}
If $\epsilon>\frac{\noiseLevel^2\rho^2}{C_{\min}^2f(\eta)}$ for some $\eta\geq\,2 + \frac{4r\rho^4(\rho^4-\rho^2+2)}{w\nu}$ and $REP\left(\eta s^*_j\right)$ holds, then the algorithm stops with (column) support size $s_j\leq(\eta-1)s^*_j$ for all $j\in\{1,2,\ldots,r\}$.
\end{lemma}

\begin{proof}
Consider the first time the algorithm reaches $k=(\eta-1)s^*_j+1$. By Lemmas~\ref{lem:errorbound} and \ref{lem:backwardstep}, we have
\begin{equation}
\begin{aligned}
\sqrt{\frac{w\nu}{r}}\sqrt{\frac{s_j-1-s^*_j}{s_j-1}}&\leq \sqrt{\frac{w\nu}{r}}\sqrt{\frac{|(\Os^{(j)}(k-1)\cup\Ob(k-1)) - (\Omega^{*(j)}_s\cup\Omega^*_b)|}{|(\Os^{(j)}(k-1)\cup\Ob(k-1)) \cup (\Omega^{*(j)}_s\cup\Omega^*_b)|}}\\ &\leq \frac{\lambda \rho}{C_{\min}\sqrt{\epsilon}} +\frac{\rho^3\sqrt{2(\rho^2-1)}} {\sqrt{|(\Os^{(j)}(k-1)\cup\Ob(k-1)) \cup (\Omega^{*(j)}_s\cup\Omega^*_b)|}}\\ &\qquad\qquad\qquad\qquad\qquad+ 2\rho^2 \sqrt{\frac{|(\Omega^{*(j)}_s\cup\Omega^*_b)-(\Os^{(j)}(k-1)\cup\Ob(k-1))|} {|(\Omega^{*(j)}_s\cup\Omega^*_b)\cup(\Os^{(j)}(k-1)\cup\Ob(k-1))|}}\\
&\leq \frac{\lambda \rho}{C_{\min}\sqrt{\epsilon}} +\frac{\rho^3\sqrt{2(\rho^2-1)}}{\sqrt{s_j-1}} + 2\rho^2 \sqrt{\frac{s^*_j}{s_j+s^*_j-1}}.\\
\end{aligned}
\nonumber
\end{equation}

Hence, we get 
\begin{equation}
f(\eta):=\frac{\sqrt{\frac{w\nu}{r}(\eta-2)}-\sqrt{\frac{2\rho^6(\rho^2-1)}{s^*_j}}}{\sqrt{\eta-1}} - \frac{2\rho^2}{\sqrt{\eta}} \leq \frac{\lambda \rho}{C_{\min}\sqrt{\epsilon}}.\\
\nonumber
\end{equation}
For $\eta\,\geq\,2 + \frac{4r\rho^4(\rho^4-\rho^2+2)}{w\nu}$, the LHS is positive and we arrive to a contradiction with the assumption on $\epsilon$.\\
\end{proof}

\begin{lemma}[General Forward Step]
For any $j\in\{1,2,\ldots,r\}$, the first time the algorithm reaches a (column) support size of $s_j$ at the beginning of the forward step, we have
\begin{equation}
\begin{aligned}
&\left|\Loss\left(\beta^{*(j)}e_j^T\right) - \Loss\left(\widehat{\beta}^{(j)}(k-1)e_j^T\right)\right|\\ &\quad\qquad\leq 2\rho C_{\min}\sqrt{\left|(\Omega^{*(j)}_s\cup\Omega^*_b)-(\Os^{(j)}(k-1)\cup\Ob(k-1))\right| \mu_s^{(k)}\epsilon}\,\left\|\beta^{*(j)} - \widehat{\beta}^{(j)}(k-1)\right\|_2.
\end{aligned}
\nonumber
\end{equation}
\label{lem:forwardstep}
\end{lemma}

\begin{proof}
According to the forward step, we have
\begin{equation} \nonumber
\begin{aligned}
\Loss\left(\widehat{\beta}(k-1)\right) - \inf_{(i,j)\notin\Os(k-1);\;\gamma\in\real}\Loss\left(\widehat{\beta}(k-1) + \gamma e_ie_j^T\right) &= \mu_s^{(k)}\epsilon.
\end{aligned}
\end{equation}
Since the loss function is separable with respect to the columns of $\beta$, for any fixed $j\in\{1,\ldots,r\}$ we have
\begin{equation} \nonumber
\begin{aligned}
\Loss\left(\widehat{\beta}^{(j)}(k-1)e_j^T\right) - \inf_{i:(i,j)\notin\Os^{(j)}(k-1);\;\gamma\in\real}\Loss\left(\widehat{\beta}^{(j)}(k-1)e_j^T + \gamma e_ie_j^T\right) &\leq \mu_s^{(k)}\epsilon.
\end{aligned}
\end{equation}

Similar to \eqref{eq:epsbound}, for any $\eta\in\real$, we have 
\begin{equation}
\begin{aligned}
&-\left|(\Omega^{*(j)}_s\cup\Omega^*_b)-(\Os^{(j)}(k-1)\cup\Ob(k-1))\right| \mu_s^{(k)}\epsilon\\ &\qquad\qquad\qquad\qquad\leq \eta\left(\Loss\left(\beta^{*(j)}e_j^T\right)-\Loss\left(\widehat{\beta}^{(j)}(k-1)e_j^T\right)\right) + \eta^2\rho^2C_{\min}^2\left\|\beta^{*(j)}-\widehat{\beta}^{(j)}(k-1)\right\|_2^2.
\end{aligned}
\nonumber
\end{equation}

Optimizing the RHS over $\eta$, we obtain
\begin{equation}
\begin{aligned}
\left|(\Omega^{*(j)}_s\cup\Omega^*_b)-(\Os^{(j)}(k-1)\cup\Ob(k-1))\right|\mu_s^{(k)}\epsilon &\geq \frac{\left(\Loss\left(\beta^{*(j)}e_j^T\right) - \Loss\left(\widehat{\beta}^{(j)}(k-1)e_j^T\right)\right)^2}{4\rho^2C_{\min}^2\,\left\|\beta^{*(j)}-\widehat{\beta}^{(j)}(k-1)\right\|_2^2}.
\end{aligned}
\nonumber
\end{equation}

This concludes the proof of the lemma.\\
\end{proof}

\bigskip

\begin{lemma}[General Error Bound]
For any $j\in\{1,2,\ldots,r\}$, the first time the algorithm reaches a (column) support size of $s_j$ at the beginning of the forward step, we have
\begin{equation}
\begin{aligned}
\left\|\beta^{*(j)}-\widehat{\beta}^{(j)}(k-1)\right\|_2 &\le
\frac{\lambda}{C_{\min}^2}\sqrt{\left|(\Omega^{*(j)}_s\cup\Omega^*_b)\cup(\Os^{(j)}(k-1)\cup\Ob(k-1))\right|}\\ &\qquad\qquad+\frac{2\rho}{C_{\min}}\sqrt{\left|(\Omega^{*(j)}_s\cup\Omega^*_b)-(\Os^{(j)}(k-1)\cup\Ob(k-1))\right| \mu_s^{(k)}\epsilon}.
\end{aligned}
\nonumber
\end{equation}
\label{lem:errorbound}	
\end{lemma}

\begin{proof}
The proof is identical to the proof of lemma~\ref{LemErrorBound} and is omitted.\\
\end{proof}

\bigskip

\begin{lemma}[General Backward Step] 
For any $j\in\{1,2,\ldots,r\}$, the first time the algorithm reaches a (column) support size of $s_j$ at the beginning of the forward step, if $s_j > s^*_j + \frac{2r\rho^6(\rho^2-1)}{\nu}$, then

\small\begin{equation}
\left\|\widehat{\beta}^{(j)}_{(\Os^{(j)}(k-1)\cup\Ob(k-1))-(\Omega^{*(j)}_s\cup\Omega^*_b)}(k-1)\right\|_2^2 \geq \left(\frac{\sqrt{\left|(\Os^{(j)}(k-1)\cup\Ob(k-1))-(\Omega^{*(j)}_s\cup\Omega^*_b)\right| w\nu }}{\rho C_{\min}\,\sqrt{r}} -\frac{\rho^2\sqrt{2(\rho^2-1)}}{C_{\min}}\right)^{\!\!\!2}\mu^{(k)} \epsilon.
\nonumber 
\end{equation}\normalsize
\label{lem:backwardstep}
\end{lemma}

\begin{proof}
Under the assumption of the lemma, the immediate previous backward step has not gone through and hence,
\begin{equation} \nonumber
\begin{aligned}
\inf_{(i,j)\in\Os(k-1)}\Loss\left(\widehat{\beta}(k) - \widehat{\beta}_i^{(j)}(k) e_ie_j^T\right) - \Loss\left(\widehat{\beta}(k)\right) &\geq \nu \mu^{(k)}\epsilon\\
\inf_{m\in\Ob(k-1)}\Loss\left(\widehat{\beta}(k) -  e_m\widehat{\beta}_m(k)\right) - \Loss\left(\widehat{\beta}(k)\right) &\geq \nu \mu^{(k)}\,w\epsilon. 
\end{aligned}
\end{equation}
Since the loss function is separable with respect to the columns of $\beta$, for a fixed $j\in\{1,2,\ldots,r\}$, we have
\begin{equation} \nonumber
\begin{aligned}
\inf_{i:(i,j)\in\Os(k-1)}\Loss\left(\widehat{\beta}^{(j)}(k)e_j^T - \widehat{\beta}_i^{(j)}(k) e_ie_j^T\right) - \Loss\left(\widehat{\beta}^{(j)}(k)e_j^T\right) &\geq \nu \mu^{(k)}\epsilon. \end{aligned}
\end{equation}

Consequently, similar to \eqref{eq:backeps}, we can show that
\begin{equation}
\begin{aligned}
&\left|(\Os^{(j)}(k-1)\cup\Ob(k-1))-(\Omega^{*(j)}_s\cup\Omega^*_b)\right| \nu \mu^{(k)}\frac{w\epsilon}{r}\\ &\quad\leq \left|\Os^{(j)}(k-1)-(\Omega^{*(j)}_s\cup\Omega^*_b\cup\Ob(k-1))\right| \nu \mu^{(k)}\epsilon + \frac{1}{r}\left|\Ob(k-1)-(\Omega^{*(j)}_s\cup\Omega^*_b)\right| w\nu \mu^{(k)}\epsilon\\ 
&\quad\leq \rho^2 C_{\min}^2 \left\|\widehat{\beta}^{(j)}_{\Os^{(j)}(k-1)-(\Omega^{*(j)}_s\cup\Omega^*_b\cup\Ob(k-1))}(k)\right\|_2^2 + \rho^2 C_{\min}^2\left\|\widehat{\beta}_{\Ob(k-1)-(\Omega^{*(j)}_s\cup\Omega^*_b)}(k)\right\|_{2,\infty}^2\\
&\quad\leq \rho^2 C_{\min}^2\left(\left\|\widehat{\beta}^{(j)}_{(\Os^{(j)}(k-1)\cup\Ob(k-1))-(\Omega^{*(j)}_s\cup\Omega^*_b)}(k-1)\right\|_2 + \left\|\Delta^{(k)}\right\|_2\right)^2,
\end{aligned}
\nonumber
\end{equation}
where, $\Delta^{(k)} = \widehat{\beta}^{(j)}_{(\Os^{(j)}(k-1)\cup\Ob(k-1))-(\Omega^{*(j)}_s\cup\Omega^*_b)}(k)-\widehat{\beta}^{(j)}(k-1)$. This entails that 

\small\begin{equation}
\begin{aligned}
&\left(\frac{\sqrt{\left|(\Os^{(j)}(k-1)\cup\Ob(k-1))-(\Omega^{*(j)}_s\cup\Omega^*_b)\right| \nu \mu_s^{(k)} w \epsilon}}{\rho C_{\min}\sqrt{r}}-\left\|\Delta^{(k)}\right\|_2\right)^{\!\!\!2}
\!\leq \left\|\widehat{\beta}_{(\Os^{(j)}(k-1)\cup\Ob(k-1))-(\Omega^{*(j)}_s\cup\Omega^*_b)}^{(j)}(k-1)\right\|_2^2.
\end{aligned}
\nonumber
\end{equation}\normalsize
Thus, it suffices to show that $\left\|\Delta^{(k)}\right\|_2\leq \frac{\rho^2}{C_{\min}}\sqrt{2(\rho^2-1) \mu_s^{(k)} \epsilon}$ since $\mu_s^{(k)}\leq \mu^{(k)}$. Notice that by our assumption on the size of the support, the first term is always larger than the second provided we can show this inequality. There are two cases: (a) if we added a single element in the previous step for which we show the above inequality, and (b) if we added a row in the previous step for which we show $\left\|\Delta^{(k)}\right\|_2\leq \frac{\rho^2}{C_{\min}}\sqrt{2(\rho^2-1) \mu_b^{(k)} \frac{w\epsilon}{r}}$. Since $\frac{w\epsilon}{r}\leq\epsilon$ and $\mu_b^{(k)}\leq\mu^{(k)}$, the result follows. We prove (a) and omit the proof of (b) since it is identical.

\bigskip
We drop the super- and sub-script $j$ for the ease of the notation in the rest of the proof. From the forward step, we have
\begin{equation} \nonumber
\Loss\left(\widehat{\beta}(k-1)\right) - \inf_{(i,j)\notin\Os(k-1),\gamma\in\real}\Loss\left(\widehat{\beta}(k-1) + \gamma e_ie_j^T\right) = \mu_s^{(k)}\epsilon.
\end{equation}
Let $(i_*,j_*,\gamma_*\neq 0)$ be the optimizer of the equation above. Now, we have
\begin{equation}
\begin{aligned}
C_{\min}^2\left\|\Delta^{(k)}\right\|_2^2 &\leq \Loss\left(\widehat{\beta}(k)_{\Os(k-1)\cup\Ob(k-1)}\right)-\Loss\left(\widehat{\beta}(k-1)\right)\\
&\leq \Loss\left(\widehat{\beta}(k)_{\Os(k-1)\cup\Ob(k-1)}\right)-\Loss\left(\widehat{\beta}(k)\right) +\Loss\left(\widehat{\beta}(k)\right)-\Loss\left(\widehat{\beta}(k-1)\right)\\
&\leq \rho^2 C_{\min}^2\left|\widehat{\beta}_{i_*}^{(j_*)}(k)\right|^2 -C_{\min}^2\left\|\Delta^{(k)}\right\|_2^2 -C_{\min}^2\left|\widehat{\beta}_{i_*}^{(j_*)}(k)\right|^2.
\end{aligned}
\nonumber
\end{equation}
Hence, $\left\|\Delta^{(k)}\right\|_2^2\leq \frac{\rho^2-1}{2}\left|\widehat{\beta}_{i_*}^{(j_*)}(k)\right|^2$ and we only need to show that $\left|\widehat{\beta}_{i_*}^{(j_*)}(k)\right| \leq\frac{2\rho^2\sqrt{\mu_s^{(k)}\epsilon}}{C_{\min}}$. Since $\left|\widehat{\beta}_{i_*}^{(j_*)}(k)\right|\leq \left|\widehat{\beta}_{i_*}^{(j_*)}(k)-\gamma_*\right|+\left|\gamma_*\right|$, we can equivalently control the latter two terms. First, by forward step construction, $C_{\min}^2\left|\gamma^*\right|^2\leq \Loss\left(\widehat{\beta}(k-1)\right)-\Loss\left(\widehat{\beta}(k-1)+\gamma_*e_{i_*}e_{j_*}^T\right)=\mu_s^{(k)}\epsilon$ and hence $\left|\gamma^*\right|\leq \frac{\sqrt{\mu_s^{(k)}\epsilon}}{C_{\min}}$. Second, we claim that $\left|\widehat{\beta}_{i_*}^{(j_*)}(k)-\gamma_*\right|\leq (2\rho^2-1)\left|\gamma_*\right|$ and we are done. 

\bigskip
In contrary, suppose $\left|\widehat{\beta}_{i_*}^{(j_*)}(k)-\gamma_*\right|^2> \left(2\rho^2-1\right)^2\left|\gamma_*\right|^2\geq \rho^2\left|\gamma_*\right|^2$. We have
\begin{equation}
\begin{aligned}
C_{\min}^2&\left|\widehat{\beta}_{i_*}^{(j_*)}(k)-\gamma_*\right|^2 > \rho^2 C_{\min}^2\left|\gamma_*\right|^2\\ &\qquad\qquad\geq \Loss\left(\widehat{\beta}(k)-\gamma_*e_{i_*}e_{j_*}^T\right) - \Loss\left(\widehat{\beta}(k)\right)\\ &\qquad\qquad\geq \Loss\left(\widehat{\beta}(k)-\gamma_*e_{i_*}e_{j_*}^T\right) - \Loss\left(\widehat{\beta}(k-1)\right) +\Loss\left(\widehat{\beta}(k-1)\right) -\Loss\left(\widehat{\beta}(k)\right)\\ &\qquad\qquad\geq C_{\min}^2\left\|\Delta^{(k)}\right\|_2^2 + C_{\min}^2\left|\widehat{\beta}_{i_*}^{(j_*)}(k)-\gamma_*\right|^2 + \grad_{(i_*,j_*)}\Loss\left(\widehat{\beta}(k-1)\right)\left(\widehat{\beta}_{i_*}^{(j_*)}(k)-\gamma_*\right)\\ &\qquad\qquad\qquad+ C_{\min}^2\left\|\Delta^{(k)}\right\|_2^2 +C_{\min}^2\left|\widehat{\beta}_{i_*}^{(j_*)}(k)\right|^2.
\end{aligned}
\nonumber
\end{equation}
This is a contradiction provided $C_{\min}^2\left|\widehat{\beta}_{i_*}^{(j_*)}(k)\right|^2 + \grad_{(i_*,j_*)}\Loss\left(\widehat{\beta}(k-1)\right)\left(\widehat{\beta}_{i_*}^{(j_*)}(k)-\gamma_*\right)\geq 0$. Later in the proof, we will show that $\sgn\left(\grad_{(i_*,j_*)}\Loss\left(\widehat{\beta}(k-1)\right)\right)=-\sgn\left(\gamma_*\right)$ and that $2C_{\min}^2|\gamma_*|\leq\left|\grad_{(i_*,j_*)}\Loss\left(\widehat{\beta}(k-1)\right)\right| \leq 2\rho^2 C_{\min}^2 |\gamma_*|$. With these, if $\frac{\widehat{\beta}_{i_*}^{(j_*)}(k)}{\gamma_*}\leq 1$, we have $\grad_{(i_*,j_*)}\Loss\left(\widehat{\beta}(k-1)\right)\left(\widehat{\beta}_{i_*}^{(j_*)}(k)-\gamma_*\right)\geq 0$ and the claim follows. Otherwise, we have $\left|\widehat{\beta}_{i_*}^{(j_*)}(k)\right|\geq \left|\widehat{\beta}_{i_*}^{(j_*)}(k)\right|-\left|\gamma_*\right|=\left|\widehat{\beta}_{i_*}^{(j_*)}(k)-\gamma_*\right|$ so that $\left|\widehat{\beta}_{i_*}^{(j_*)}(k)\right|\geq 2\rho^2\left|\gamma_*\right|$ and hence,
\begin{equation}
\begin{aligned}
C_{\min}^2\left|\widehat{\beta}_{i_*}^{(j_*)}(k)\right|^2 &+ \grad_{(i_*,j_*)}\Loss\left(\widehat{\beta}(k-1)\right)\left(\widehat{\beta}_{i_*}^{(j_*)}(k)-\gamma_*\right)\\ & \geq 2\rho^2 C_{\min}^2 \left|\gamma_*\right|\left|\widehat{\beta}_{i_*}^{(j_*)}(k)-\gamma_*\right| - 2\rho^2 C_{\min}^2\left|\gamma_*\right|\left|\widehat{\beta}_{i_*}^{(j_*)}(k)-\gamma_*\right|\\
&= 0.
\end{aligned}
\nonumber
\end{equation}

\bigskip
To get the claimed properties of $\grad_{(i_*,j_*)}\Loss\left(\widehat{\beta}(k-1)\right)$, note that
\begin{equation}
\begin{aligned}
C_{\min}^2\left|\gamma_*\right|^2 &\leq \Loss\left(\widehat{\beta}(k-1)\right) - \Loss\left(\widehat{\beta}(k-1)+\gamma_*e_{i_*}e_{j_*}^T\right)\\
&\leq -C_{\min}^2\left|\gamma_*\right|^2 -\grad_{(i_*,j_*)}\Loss\left(\widehat{\beta}(k-1)\right) \gamma_*\,,
\end{aligned}
\nonumber
\end{equation}
and hence $\sgn\left(\grad_{(i_*,j_*)}\Loss\left(\widehat{\beta}(k-1)\right)\right)=-\sgn\left(\gamma_*\right)$ and $2C_{\min}^2|\gamma_*|\leq\left|\grad_{(i_*,j_*)}\Loss\left(\widehat{\beta}(k-1)\right)\right|$. Also, we can establish
\begin{equation}
\begin{aligned}
\rho^2 C_{\min}^2\left|\gamma_*\right|^2 &\geq \Loss\left(\widehat{\beta}(k-1)\right) - \Loss\left(\widehat{\beta}(k-1)+\gamma_*e_{i_*}e_{j_*}^T\right)\\
&\geq -\rho^2 C_{\min}^2\left|\gamma_*\right|^2 -\grad_{(i_*,j_*)}\Loss\left(\widehat{\beta}(k-1)\right) \gamma_*\,.
\end{aligned}
\nonumber
\end{equation}
Since $-\grad_{(i_*,j_*)}\Loss\left(\widehat{\beta}(k-1)\right) \gamma_*\geq 0$, we can conclude that $\left|\grad_{(i_*,j_*)}\Loss\left(\widehat{\beta}(k-1)\right)\right|\leq 2\rho^2 C_{\min}^2|\gamma_*|$. This concludes the proof of the lemma.\\
\end{proof}

\section{Proof of Corollary~\ref{cor:random}}
The result follows from the following two lemmas.

\begin{lemma}
Given the sample complexity $n_j\geq c_5\log(rp)$ for some constant $c_5$ and all $j\in\{1,2,\ldots,r\}$, we have
\begin{equation}
\begin{aligned}
\lambda:=\max_j\|\nabla^{(j)}\|_\infty\leq c_4\sqrt{\frac{\log(rp)}{n}}
\end{aligned}
\nonumber
\end{equation}
with probability at least $1-c_6\exp(-c_7n)$ for some positive constants $c_5,c_6$ and $c_7$.
\label{lem:Gradient-Concentration}
\end{lemma}

The proof follows from Lemma 5 in \cite{WainwrightLasso}. We state our theoretical result in terms of $\lambda$ for the sake of generality. This parameter can be replaced with any upper-bound on $\nabla^{(j)}$ and our guarantee still holds.\\

\begin{lemma}
If each row of the design matrix $X^{(j)}\in\real^{n\times p}$ is distributed as $\mathcal{N}(0,\Sigma^{(j)})$ and $\Sigma^{(j)}$ satisfies $REP(s_j)$, then for any small $\theta>0$, the matrix $Q^{(j)}=X^{(j)}X^{(j)T}$ satisfies
\begin{equation}
\begin{aligned}
(1-\theta)C_{\min}\|\delta\|_2\leq\|Q^{(j)}\delta\|_2\leq(1+\theta)\rho C_{\min}\|\delta\|_2,
\end{aligned}
\end{equation}
for all $\|\delta\|_0\leq s_j$, with probability $1-c_8\exp(-c_9n)$ provided that $n_j\geq c_{10}(\theta)\, s_j\log(p)$, where $c_8-c_{10}$ are constants independent of $(n_j,s_j,p)$.
\label{lem:RSC-RSS-Concentration}
\end{lemma}
The proof follows from Lemma 9 (Appendix K) in \cite{WainwrightLasso}. This lemma shows that for Gaussian design matrices, $REP(s_j)$ is satisfied with high probability for $\mathcal{O}(s_j\log(p))$ samples.\\

\end{document}